\documentclass[12pt]{article}

\usepackage[margin=1in]{geometry}
\usepackage{authblk}
\usepackage{natbib}

\usepackage[utf8]{inputenc} 
\usepackage[T1]{fontenc}    
\usepackage{hyperref}       
\usepackage{url}            
\usepackage{booktabs}       
\usepackage{amsfonts, amsmath, amsthm}       
\usepackage{nicefrac}       
\usepackage{microtype}      
\usepackage{xcolor}         
\usepackage{mathrsfs}
\usepackage{graphicx}
\usepackage{subcaption}
\usepackage{color,soul}
\usepackage[normalem]{ulem}
\usepackage{wrapfig}

\newcommand{\R}{\mathbb{R}}

\newcommand{\KL}{\mathrm{KL}}
\newcommand{\ELBO}{\mathcal{L}_{\mathrm{ELBO}}}
\newcommand{\E}{\mathbb{E}}
\newcommand{\N}{\mathcal{N}}
\newcommand{\bw}{\mathbf{w}}
\newcommand{\bmu}{\boldsymbol{\mu}}
\newcommand{\bsig}{\boldsymbol{\sigma}}
\newcommand{\brho}{\boldsymbol{\rho}}

\newtheorem{proposition}{Proposition}

\title{Bayesian 3D Steerable CNNs: Enabling Equivariance and Uncertainty Quantification Simultaneously}

\author[1]{Abhishek Keripale} 
\author[3]{Ponkrshnan Thiagarajan} 
\author[1,2]{Susanta Ghosh\thanks{ Corresponding author: \texttt{susantag@mtu.edu}}}
\affil[1]{Department of Mechanical and Aerospace Engineering, Michigan Technological University, Houghton, MI 49931}
\affil[2]{The Center for Artificial Intelligence at the Institute of Computing and Cybersystems, Michigan Technological University, USA}
\affil[3]{Hopkins Extreme Materials Institute, Johns Hopkins University, Baltimore, MD 21218}

\date{}

\begin{document}

\maketitle

\begin{abstract}

Steerable convolutional neural networks (Steerable-CNNs) guarantee SE(3)-equivariance by parameterizing kernels as linear combinations of steerable basis functions, but their deterministic nature precludes uncertainty quantification—limiting their use in settings where confidence estimates are essential. We propose a Bayesian Steerable-CNN that places posterior distributions over the basis coefficients, yielding stochastic kernels while preserving equivariance exactly. The loss function of the model is obtained via variational inference and minimized by Bayes-by-Backpropagation. The framework admits a decomposition of predictive uncertainty into epistemic and aleatoric components. Empirically, the model attains competitive classification accuracy alongside an expected calibration error of 0.0263 and outperforms its deterministic counterpart by up to 6.17\% under distributional shift induced by additive Gaussian noise. Furthermore, we leverage the model's uncertainty estimates to enhance its performance significantly, achieving a notable gain—approximately 4\% higher accuracy across 84\% of the test dataset. A statistically significant negative correlation between epistemic uncertainty and prediction error confirms that the learned posterior variance is semantically meaningful. The framework unifies Bayesian uncertainty quantification with the inductive bias of equivariant CNNs.

\end{abstract}

\section{Introduction} \label{sec:intro}
Encoding the inherent symmetries of the data in a Machine Learning (ML) models is beneficial as they make the model more robust, endows with inductive bias to capture and respect the symmetry patterns in the data irrespective of its orientation. This property due to which the ML models, by construction, respect the symmetries of the data is referred to as equivariance and is obtained by introducing some symmetry constraints in the ML architecture. A neural network is equivariant to a symmetry group if transforming the input by a group element — such as a 3D rotation — produces a correspondingly transformed output, rather than an arbitrary one. This property is achieved by constraining convolutional kernels to a group-theoretic basis whose elements individually satisfy the symmetry constraint (referred to as the intertwiner or steerability constraint), ensuring that any learned kernel is equivariant by construction. The resulting inductive bias reduces the effective hypothesis space, enables weight sharing across orientations, and yields strong generalization on inherently symmetric domains such as 3D shape recognition and molecular property prediction.

Steerable CNNs~\citep{weiler20183d,cesa2022program,weiler2019general,cohen2016steerable} guarantee equivariance to compact groups $G \leq \mathrm{O}(3)$ by constraining convolutional kernels to a group-theoretic basis, and have demonstrated strong performance on 3D shape recognition~\citep{cesa2022program} and molecular property prediction~\citep{thomas2018tensor,anderson2019cormorant}. However, these architectures are entirely \emph{deterministic}: a single forward pass produces a single prediction with no associated confidence in the prediction. Unfortunately, a model’s predictions can be uncertain. This uncertainty arises from inherent randomness in the data (aleatoric uncertainty) or from a lack of knowledge or data (epistemic uncertainty). Uncertainty in a model’s prediction informs how much the predicted values can vary. 
The inability to quantify uncertainty is a critical limitation for a model in safety-sensitive domains such as medical image analysis and autonomous systems. Adding stochasticity to the machine learning models not only provides uncertainty quantification, but also provides an implicit regularization to weights which improves generalization to the unseen data compared to deterministic models~\citep{blundell2015weight}.

There are several modeling approaches for developing stochastic models and performing uncertainty quantification (UQ); among them, Bayesian methods are widely used due to their ability to combine prior beliefs with observed data via Bayes’ theorem and update predictive probabilities.
Bayesian deep learning \citep{blundell2015weight,gal2016dropout,mackay1992practical,shridhar2019comprehensive} follows this approach, it treats network weights as random variables and performs approximate posterior inference, enabling principled decomposition of predictive uncertainty into epistemic and aleatoric components \citep{kendall2017uncertainties,kwon2018uncertainty,thiagarajan2021explanation}. Such decomposition provides practitioners with actionable diagnostics: high epistemic uncertainty signals insufficient training data or model capacity, while high aleatoric uncertainty indicates inherent ambiguity in the input---information that is entirely inaccessible to point-estimate models.

A key challenge for developing a Bayesian model for steerable CNNs is that naively extending steerable CNNs to a Bayesian setting will break their equivariance property, since, randomly sampled convolution kernels will violate the intertwiner (or steerability) constraint. Thus, even if the mean kernel matrices are equivariant, stochastic realizations of kernel matrices will not be equivariant.

The intertwiner (or steerability)  constraint is a property of the steerable basis functions, and the equivariant kernels are linear combinations of these basis. We develop stochastic equivariant kernels by considering stochastic scalar coefficients in this linear combination while keeping the steerable basis functions deterministic. We develop a Bayesian model by considering the variational posterior of the scalar coefficients following the Bayes by backpropagation approach of \cite{blundell2015weight}. Thus, every kernel obtained by sampling form these posteriors of the coefficients maintains equivariance.

Our main contributions are:
\begin{itemize}
  \item A novel Bayesian formulation for steerable CNNs is proposed that obtains the posterior distribution of the coefficients of the kernel basis. The evidence lower bound loss function for steerable CNNs is obtained and optimized.  
  \vspace{-1pt}
  \item A proof that the proposed Bayesian model maintains $\mathrm{SE}(3)$-equivariance. 
  \item Empirical validation of the model on ModelNet10, achieving competitive accuracy while providing calibrated and decomposable (into epistemic + aleatoric) uncertainty estimates.
\end{itemize}

\section{Related Work} \label{sec:related}

In the pursuit of achieveing equivariance in a systematic manner, among various approaches, group-equivariant CNNs and steerable CNNs are the two prominent paradigms that exploit group and representation theoretic framework in embedding the equivariance in the model architecture itself.
Group-equivariant CNNs~\citep{cohen2016group} achieve equivariance by convolving over the group domain using the regular representation, imposing a one-to-one relation between group structure and weight-sharing. On the other hand, steerable CNNs~\citep{cohen2016steerable,weiler2019general} provide a strictly more general framework that supports arbitrary $G$-representations, including the irreducible representations (irreps) and their direct sum, for any compact group.

In parallel to the deterministic equivariant ML models there are significant progress in stochastic ML models. 
Among various ML models that provide uncertainty quantification (UQ), Bayesian methods have gained prominence due to their ability to incorporate prior knowledge. Bayesian deep learning (\citet{blundell2015weight}) models aims to infer the posterior distribution over weights using Bayes’ rule. This posterior is approximated by a simple, tractable distribution called the variational posterior—an approach known as variational inference, which casts Bayesian inference as an optimization problem. In this approach the divergence between the true posterior and the variational posterior is minimized, yielding the objective (loss) function to be optimized. 

The metric chosen to represent the divergence between the true posterior and the variational posterior is also extensively studied \citet{thiagarajan2025jensen,dieng2017variational,deasy2020constraining,li2016renyi,wan2020f}. \citet{gal2016dropout} showed that dropout at test time approximates Bayesian inference. 
One key advantage is that the uncertainty can be uniquely decomposed into epistemic and aleatoric components follows~\citet{kendall2017uncertainties} and~\citet{kwon2018uncertainty}.
UQ for ML model that consider other important aspects such as geometric learning have also been studied, for instance \citet{postels2019sampling} and \citet{liu2020simple} have explored uncertainty in point-cloud networks. 
There are equivariant probabilistic models such as the normalizing flows~\citep{kohler2020equivariant} and equivariant diffusion models~\citep{hoogeboom2022equivariant} impose equivariance on generative models. 
However, the formal equivariance guarantees and UQ are not studied in a single discriminative model. 

To the best of our knowledge, this is the first work that integrates stochastic inference and steerable convolutions, and exploits it to develop a Bayesian steerable convolution model.

\section{Background} \label{sec:background}
 
\subsection{G-Steerable Convolutions}

At layer $n$, a standard CNN produces $K_n$ feature maps, each modeled as a scalar-valued function $f_k : \mathbb{R}^3 \to \mathbb{R}$. Collectively, these can be represented as a single vector-valued function $f : \mathbb{R}^3 \to \mathbb{R}^{K_n}$, which assigns to each spatial location $x \in \mathbb{R}^3$ a feature vector $f(x) \in \mathbb{R}^{K_n}$. We define the feature space at layer $n$ as $\mathcal{F}_n = \{ f : \mathbb{R}^3 \to \mathbb{R}^{K_n}  \}$, the space of square-integrable vector-valued functions on $\mathbb{R}^3$ where $f$ is compact and continuous. While this formulation is sufficient for standard CNNs, it treats all feature channels as independent scalar coordinates and carries no information about how the feature vector $f(x)$ should transform when the input undergoes a geometric transformation. This is precisely the structure that must be made explicit in order to define equivariant CNNs.
\vspace{-4pt}
\begin{wrapfigure}{l}{0.55\textwidth}
\centering
\includegraphics[width=1\linewidth]{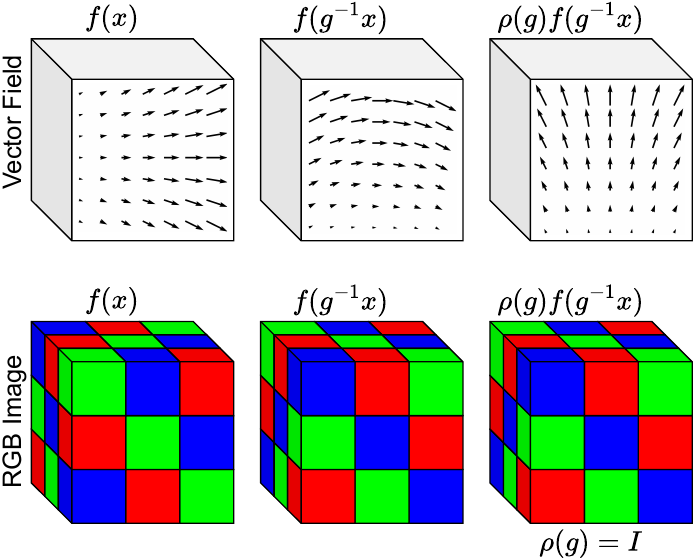}
\vspace{-10pt}
\caption{Transformation of feature fields under a rotation $g$. \textbf{Top row:} a vector field, where the transformation proceeds in two stages — vectors are first moved to new positions, followed by a rotation of the vectors itself. \textbf{Bottom row:} an RGB image, where rotation affects only spatial positions, while the channel space remains unchanged carrying trivial representation $\rho(g)=I$.}
\label{fig:transformation_of_feature_fields}
\end{wrapfigure}

A feature field in equivariant CNNs $f : \mathbb{R}^3 \to \mathbb{R}^{c}$ is said to be of geometric type $\rho$ if it transforms under the joint action of a rotation $g \in G$ and a translation $t \in \mathbb{R}^3$ according to the induced representation
\vspace{-1pt}
\begin{equation*}
    [\pi_{\rho}(g,t)\, f](x) \;=\; \rho(g)\, f\!\bigl(g^{-1}(x - t)\bigr)
\end{equation*}
\vspace{-1pt}
where $\rho : G \to \mathrm{GL}(\mathbb{R}^{c})$ is a group representation acting on the fiber $\mathbb{R}^{c}$ at each spatial location. The transformation therefore proceeds in two stages: the spatial argument is moved to the pre-image location $g^{-1}(x-t)$, and the feature vector at that location is simultaneously transformed by the linear map $\rho(g)$. This two-stage transformation of feature fields is illustrated in Fig.~\ref{fig:transformation_of_feature_fields} for a vector field and an RGB image. In the case of a vector field, the representation $\rho$ is the standard rotation matrix, so the channels are coupled or mixed and transform jointly under rotations. In the case of an RGB image, $\rho(g) = I$ for all $g \in G$, i.e. the trivial representation, and the channels remain unchanged while only the spatial positions are affected (Refer to \cite{weiler20183d} for more details). Two feature fields may therefore have the same dimensionality $c$ yet exhibit fundamentally different transformation behaviour depending solely on the choice of $\rho$, a distinction that is entirely invisible to standard CNNs.

This observation motivates endowing each feature field with a geometric type, specified by the representation $\rho$, rather than treating it as an arbitrary element of $\mathbb{R}^c$. A convolutional layer mapping a $\rho_{\mathrm{in}}$-type input field to a $\rho_{\mathrm{out}}$-type output field is then equivariant if and only if convolution commutes with the induced group action on both fields, i.e.
\vspace{-1pt}
\begin{equation}
    K \star \bigl[\pi_{\rho_{\mathrm{in}}}(g,t)\, f\bigr]
    \;=\;
    \pi_{\rho_{\mathrm{out}}}(g,t)\, \bigl[K \star f\bigr],
\end{equation}
for all $g \in G$ and $t \in \mathbb{R}^3$. This condition places a hard constraint on the convolutional kernel $K$, known as the steerability constraint, which characterises the subspace of kernels that can serve as equivariant maps between input and output feature fields with their geometric types $\rho_{\mathrm{in}}:G \to \R^{d_{in} \times d_{in}}$ and $\rho_{\mathrm{out}}:G \to \R^{d_{out} \times d_{out}}$ respectively. 

The steerability constraint (also referred to as intertwiner) is given as $\forall g \in G,\, \mathbf{x} \in \mathbb{R}^n$,
\vspace{-1pt}
\begin{equation} \label{eq:kernel_constraint}
    K(g \cdot \mathbf{x}) = \rho_\text{out}(g)\, K(\mathbf{x})\, \rho_\text{in}(g)^{-1}
\end{equation}
For computational convenienence vectorized form of Eq. \ref{eq:kernel_constraint} is used with $K(\cdot)=\text{vec}(K(\cdot)): \mathbb{R}^n \to \mathbb{R}^{d_{out}.d_{in}}$ to parametrize a kernel satisfying the steerability condition. Thus, $\forall g \in G,\, \mathbf{x} \in \mathbb{R}^n$: $K(g \cdot\mathbf{x}) = \big[ (\rho_{\text{in}} \otimes \rho_{\text{out}})(g)\big] K(\mathbf{x})$
where, $(\rho_{\text{in}} \otimes \rho_{\text{out}})(g)$ is a matrix resulting from Kronecker product of two matrices $\rho_{\text{in}}(g)$ and $\rho_{\text{out}}(g)$. Refer to \cite{cesa2022program} for more information.

Since any representation of a compact group decomposes into a direct sum of irreducible
representations (irreps) by the Peter--Weyl theorem, it suffices without loss of generality
to consider $\rho_{\mathrm{in}} = \rho_l$ and $\rho_{\mathrm{out}} = \rho_J$ to be irreps,
indexed by $l, J \in \hat{G}$~\citep{cesa2022program, weiler2019general,veefkind2024probabilistic}. The tensor
product $\rho_l \otimes \rho_J$ is itself a $G$-representation but not necessarily
irreducible; it admits a further decomposition into irreps via the Clebsch--Gordan
decomposition~\citep{cesa2022program, veefkind2024probabilistic}:
\vspace{-12pt}
\begin{equation}
    \label{eq:cg_decompose}
    (\rho_{l} \otimes \rho_{J})(g)=[\text{CG}^{lJ}]^T \bigg( \bigoplus_{j} \bigoplus_s^{[j(lJ)]}\rho_j(g) \bigg) \text{CG}^{lJ}
\end{equation}
\noindent where, $[j(lJ)]$ is the multiplicity of the irrep $\rho_j$, and $\text{CG}^{lJ}$ is the change of basis matrix.
As the direct sum takes the block diagonal structure, the blocks of $\text{CG}^{lJ}$ act only on corresponding specific irrep. For example, $\text{CG}_{j_1}^{lJ}$ acts on $\rho_{j_1}$ with its respective multiplicity and so on. Thus, following Theorem 2 of \cite{veefkind2024probabilistic}), the Eq. \ref{eq:cg_decompose} can also be written with regular summation instead of direct sum  as $(\rho_{l} \otimes \rho_{J})(g)=\sum_j \sum_s^{[j(lJ)]} \text{CG}_s^{j(lJ)}\rho_j(g)$ with the coefficients $\text{CG}_s^{j(lJ)} \in \mathbb{R}^{d_j \times d_l d_J}$ being denoted as the $s$-th occurence of $\rho_j$ (Refer Sec. 2.2.1: Theorem 2 of \cite{veefkind2024probabilistic}). These $\text{CG}_s^{j(lJ)}$ matrices are the resulting Clebsch-Gordan coefficients.
 \footnote{The block associated with the $s$-th occurence of $\rho_j$ is denoted as $\text{CG}_s^{j[lJ]} \in \mathbb{R}^{d_j \times d_l d_J}$ \cite{cesa2022program}.}

The solution to steerability constraint in Eq. \ref{eq:kernel_constraint} is to expand the kernel as a linear combination of a finite dimensional linear basis subspace, called steerable basis. For the special Euclidean group $\mathrm{SE}(3)$, these basis are the Wigner D matrices which are the irreducible representations of the $\mathrm{SO}(3)$. The convolution kernel expanded in steerable basis is given as,
\vspace{-5pt}
\begin{equation} \label{eq:kernel_expansion}
    K(\mathbf{x}; \mathbf{w}) = \sum_{b=1}^B w_b \kappa_b(\mathbf{x})
\end{equation}
\noindent where $\{ \kappa_b\}_{b=1}^B$ are the steerable basis elements and $\mathbf{w} = (w_1, \ldots, w_B) \in \mathbb{R}^B$ are the learnable basis coefficients. The basis elements $\kappa_b(\mathbf{x})$ are fixed and computed once before training. They are equivariant by construction following Eq. \ref{eq:kernel_constraint}. 
\vspace{-3pt}
\begin{equation} \label{eq:kernel_basis_equivariance}
    \kappa_b(g\cdot\mathbf{x}) =  \rho_\text{out}(g)\kappa_b(\mathbf{x})\rho_\text{in}(g)^{-1} \quad \forall g \in G, \forall b
\end{equation}
\vspace{-5pt}
Thus, equivariance is preserved by construction, regardless of the values of $\mathbf{w}$.

\section{Proposed Bayesian Steerable CNN} \label{sec:methodology}
 
\subsection{Variational Inference for Steerable CNN}
The core building block of our approach is the stochastic steerable $\R^3$-convolution. This stochastic convolution is formulated by replacing the deterministic weight vector $\bw \in \R^B$ of the steerable convolution layer~\citep{cesa2022program} with stochastic scalar coefficients of steerable basis function. The introduction of stochasticity preserves the equivariance of the steerable convolution as shown below. 

\begin{proposition}[Preservation of Equivariance]
\label{prop:equivariance}
Let $\{\kappa_b\}_{b=1}^{B}$ be a $G$-steerable basis set satisfying Eq. ~\eqref{eq:kernel_constraint}. Let $q(\bw; \lambda)$ be any probability distribution on $\R^B$. Then for $\bw^{(s)} \sim q(\bw;\lambda)$, the stochastic kernel $K(x;\, \bw^{(s)}) = \sum_{b=1}^{B} w_b^{(s)}\, \kappa_b(x)$ satisfies the steerability constraint Eq.~\eqref{eq:kernel_constraint}.
\end{proposition}
 
\textbf{Proof.}\;
For any $\bw^{(s)} \in \R^B$, $g \in G$, and $x \in \R^n$:
\begin{align}
  K(g \cdot x;\, \bw^{(s)})
    &= \textstyle\sum_{b} w_b^{(s)}\, \kappa_b(g \cdot x) \notag\\
    &= \textstyle\sum_{b} w_b^{(s)}\, \rho_{\mathrm{out}}(g)\, \kappa_b(x)\,
       \rho_{\mathrm{in}}(g)^{-1}
       \tag{by Eq.~\ref{eq:kernel_basis_equivariance}} \\
    &= \rho_{\mathrm{out}}(g) \bigg( \textstyle\sum_{b} w_b^{(s)}\, \kappa_b(x)\bigg)\,
       \rho_{\mathrm{in}}(g)^{-1} \notag \\
    &= \rho_{\mathrm{out}}(g)\,
       K(x;\,\bw^{(s)})\,
       \rho_{\mathrm{in}}(g)^{-1} \notag
\end{align}
Thus, $K$ satisfies the constraint in Eq. \ref{eq:kernel_constraint}. 
Since this holds for any arbitrary $\bw^{(s)} \sim q$,  it holds for all $\bw^{(s)}$ sampled from a distribution $q$.
 
\smallskip
 
This result is the theoretical foundation of our approach: stochasticity in the coefficient space \emph{cannot} break equivariance, because the basis---not the coefficients---is what enforces the intertwiner constraint. No special architecture modifications or equivariance-aware sampling are required.

To learn the stochastic steerable convolution, a mean-field Gaussian variational posterior (Eq.~\ref{eq:posterior}) over the $B$ steerable basis coefficients is assumed following the Bayes by backpropagation  framework~\citep{blundell2015weight,shridhar2019comprehensive}. 

\vspace{-4pt}
\begin{equation}\label{eq:posterior}
    q(\bw \mid \boldsymbol{\theta})
    = \prod_{b=1}^{B} \N\!\left(w_b;\, \mu_b,\, \sigma_b^2\right),
    \qquad
    \sigma_b = \log\!\left(1 + e^{\rho_b}\right),
\end{equation}

where the variational parameters $\boldsymbol{\theta} = \{\bmu, \brho\}$ are learned by maximising the Evidence Lower Bound (ELBO):
\vspace{-4pt}
\begin{equation}\label{eq:elbo}
    \ELBO(\boldsymbol{\theta})
    = \E_{q(\bw|\boldsymbol{\theta})}\!\left[\log p(\mathcal{D} \mid \bw)\right]
    - \KL\!\left[q(\bw \mid \boldsymbol{\theta}) \;\|\; p(\bw)\right].
\end{equation}

Since both the prior $p(\bw) = \N(\mathbf{0},\, \sigma_p^2 \mathbf{I})$ and the posterior are Gaussian, the KL divergence admits a closed-form expression 
\vspace{-6pt}
\begin{equation}\label{eq:kl}
     \KL\!\left[q \;\|\; p\right]
     = \frac{1}{2} \sum_{b=1}^{B}
     \left[
         2\log\frac{\sigma_p}{\sigma_b}
         - 1
         + \frac{\sigma_b^2}{\sigma_p^2}
         + \frac{\mu_b^2}{\sigma_p^2}
     \right].
\end{equation}

We adopt a hybrid estimator: the KL term is computed analytically and the log-likelihood is estimated via a single Monte Carlo sample drawn through the reparameterisation trick~\citep{kingma2013auto},
\vspace{-4pt}
\begin{equation}
\label{eq:reparametrization}
    \bw = \bmu + \boldsymbol{\varepsilon} \odot \bsig \;\;; \boldsymbol{\varepsilon} \sim \N(\mathbf{0}, \mathbf{I})
\end{equation}

Each layer maintains learnable vectors $\bmu \in \R^B$ and $\brho \in \R^B$, initialised from $\N(0, 0.1^2)$ and $\N(-3, 0.1^2)$ respectively, with prior $p(\bw) = \N(\mathbf{0},\, \sigma_p^2 \mathbf{I})$, $\sigma_p = 0.1$ following \citep{shridhar2019comprehensive}. During training, a weight sample is drawn via the reparameterisation trick in Eq. \ref{eq:reparametrization} and expanded into a full convolutional filter via the steerable basis following Eq. \ref{eq:kernel_expansion} with $\{\kappa_b\}_{b=1}^{B}$ as the Wigner--Eckart basis for $\mathrm{SO}(3)$ and frozen thereafter. During inference, the mean weights $\bw = \bmu$ are used by default, and stochastic samples are drawn only when uncertainty estimation is requested.

In practice, we use the per-batch ELBO loss for the model, which is defined as:
\vspace{-5pt}
\begin{equation}\label{eq:loss}
    \mathcal{L}_{\mathrm{train}} = \mathcal{L}_{\mathrm{CE}}(\hat{y}, y) + \beta \cdot \KL_{\mathrm{total}},
\end{equation}
where the log-likelihood in \eqref{eq:elbo} is modeled as the cross-entropy $\mathcal{L}_{\mathrm{CE}}$, evaluated at a single reparameterised weight sample, and $\KL_{\mathrm{total}} = \sum_{\ell} \KL^{(\ell)}$ sums the closed-form KL divergence (Eq. \ref{eq:kl}) over all Bayesian layers, and $\beta$ is the KL annealing coefficient arising from minibatch training following \citet{blundell2015weight}. The value of $\beta$ is taken as  $10^{-5}$ in the experiments unless stated otherwise. 
 
\subsection{Uncertainty Quantification}
The Bayesian formulation provides principled uncertainty estimation that is entirely inaccessible to deterministic models. At inference, $N$ 
stochastic forward passes are performed, each drawing $\bw^{(n)} \sim q(\bw|\boldsymbol{\theta})$. Let $\hat{p}_n = \mathrm{softmax}(f_{\bw^{(n)}}(\mathbf{x})) \in \Delta^{C-1}$ and $\bar{p} = \frac{1}{N}\sum_{n} \hat{p}_n$. Following~\citet{kwon2018uncertainty,kendall2017uncertainties,thiagarajan2021explanation}, the total predictive uncertainty decomposes into two parts aleatoric and epistemic, as:
\vspace{-4pt}
\begin{equation}
\label{eq:uncertainty}
    \mathrm{Var}_{q(\bw|\boldsymbol{\theta})}[p(\hat{y}|\hat{x}, \bw)] = \underbrace{\frac{1}{N}\sum_{n=1}^{N} \left(\mathrm{diag}(\hat{p}_{n})-\hat{p}_{n}\hat{p}_{n}^{T}\right)}_{\text{aleatoric } (\Sigma_{\mathrm{ale}})} + 
    \underbrace{\frac{1}{N}\sum_{n=1}^{N} (\hat{p}_{n}-\bar{p})(\hat{p}_{n}-\bar{p})^T}_{\text{epistemic } (\Sigma_{\mathrm{epi}})},
\end{equation}
\vspace{-12pt}

with scalar summaries $u_{\mathrm{ale}} = \mathrm{tr}(\Sigma_{\mathrm{ale}})$ and $u_{\mathrm{epi}} = \mathrm{tr}(\Sigma_{\mathrm{epi}})$. 
This decomposition separates uncertainty due to inherent input data ambiguity (i.e., the aleatoric component) from uncertainty arising from limited or unrepresentative training data (i.e., the epistemic component)—two distinct sources that require different responses.
In the present work, we have considered $N = 30$.

\paragraph*{Expected Calibration Error (ECE):} To measure the discrepancy between predicted confidence and empirical accuracy, we compute the Expected Calibration Error (ECE) \cite{guo2017calibration}. ECE is calculated by partitioning predictions into equally spaced bins based on confidence scores. Each bin $B_m$ is the set of samples whose prediction confidence falls between the interval $I_m = \big(\frac{m-1}{M}, \frac{m}{M}\big ]$, where $M$ is total number of equal width bins. It computes the weighted average of the absolute difference between the average confidence and accuracy within each bin: 
$ECE = \sum_{m=1}^{M} \frac{|B_m|}{n} \left| \text{acc}(B_m) - \text{conf}(B_m) \right|$
where $n$ is the total number of samples, $|B_m|$ is the number of samples in bin $m$, \text{acc}($B_m$) is the empirical accuracy of predictions in bin $m$, and \text{conf}($B_m$) is the average predicted confidence in that bin. Lower ECE indicates better calibration, i.e., predicted probabilities more closely match observed frequencies.

\subsection{Equivariance Verification} \label{sec:equivariance_verification}
 
Equivariance is verified empirically by applying a rotation $R_j$ to the input voxel. Nine such rotations are applied to validate the Equivariance: $\pi/2$, $\pi$, $3\pi/2$ about each mutually orthogonal axes, parallel to the grids and measuring the mean $\mathbb{L}_2$ difference between original and rotated logit outputs:
\vspace{-4pt}
\begin{equation}
\label{eq:equiv_error}
    \epsilon_{\mathrm{eq}} = \frac{1}{|\mathcal{R}|} \sum_{R_j \in \mathcal{R}}\; \frac{1}{|\mathcal{D}_{\mathrm{test}}|} \sum_{i}\, \bigl\|f(\mathbf{x}_i) - f(R_j\,\mathbf{x}_i)\bigr\|_2,
\end{equation}
\vspace{-10pt}

where $f$ denotes the model output using posterior mean weights on an input $\mathbf{x}$, $\mathcal{R}$ denotes the set of rotations $\pi/2$, $\pi$, $3\pi/2$ about each mutually orthogonal axes, parallel to the grids, $|\mathcal{D}|$ denotes the size of the dataset. Since the final pooling retains only the $\ell=0$ component --- the invariant features, $\epsilon_{\mathrm{eq}}$ is expected to be near machine precision, where the $\ell$ denotes order of the basis. The detailed explanation can be found in Appendix \ref{sec:invariance_verification}.

\section{Experiments}\label{sec:Experiments}
All experiments were implemented in Python using PyTorch~\citep{paszke2019pytorch}. The equivariant neural network architecture was entirely built using the \texttt{escnn} library~\citep{cesa2022program}, which provides a principled framework for constructing steerable CNNs with exact equivariance to the rotation group $SO(3)$. Bayesian inference via Bayes-by-Backpropagation was implemented on top of this framework to enable uncertainty quantification. All models were trained and evaluated on the ModelNet10 dataset~\citep{wu20153d} using NVIDIA A100-SXM4-80GB GPU on the NERSC Perlmutter HPC facility.

\subsection{ModelNet10 classification} \label{subsec:modelnet10_classification}
We first considered the ModelNet10 dataset for classification which is introduced by \citet{wu20153d}. This is a 10 class dataset consisting of 3D CAD models of desk, night stand, bed, toilet, etc. It contains total 4899 CAD models stored as polygon meshes with official $80\% -20\%$ split for training (3991 CAD models) and testing (908 CAD models) respectively. We further split the training data, allocating about $10\%$ for validation. As a result, our training set contains 3592 samples, while validation set includes 399 samples. A few training samples of the ModelNet10 dataset has been shown in the Appendix.

\section{Results}
\subsection{Training Dynamics}

The evolution of the classification accuracy over 1000 epochs on the ModelNet10 benchmark dataset is shown in Figure~\ref{fig:accuracy_plot}. Both training and validation accuracy rise sharply in the first 40 epochs, with the model reaching $93.68\%$ training accuracy and $89.22\%$ validation accuracy by convergence, indicating that the Bayesian SE(3)-equivariant architecture learns effectively. The small generalisation gap between training and validation accuracy throughout training suggests that the KL regularisation term successfully constrains the posterior and prevents the model from collapsing to a point estimate.

The model achieved a test accuracy of 86.34\%, indicating good classification performance. Its Expected Calibration Error (ECE) of 0.0263 suggests that the predicted confidences are quite well aligned with empirical correctness, with an average calibration gap of about 6\% points. Furthermore, the low ECE indicates that the posterior predictive probabilities are well calibrated, meaning the model’s confidence scores are broadly consistent with observed accuracy across the test set.

\begin{figure}[h]
  \centering
  \begin{subfigure}[b]{0.49\linewidth}
    \centering
    \includegraphics[width=\linewidth]{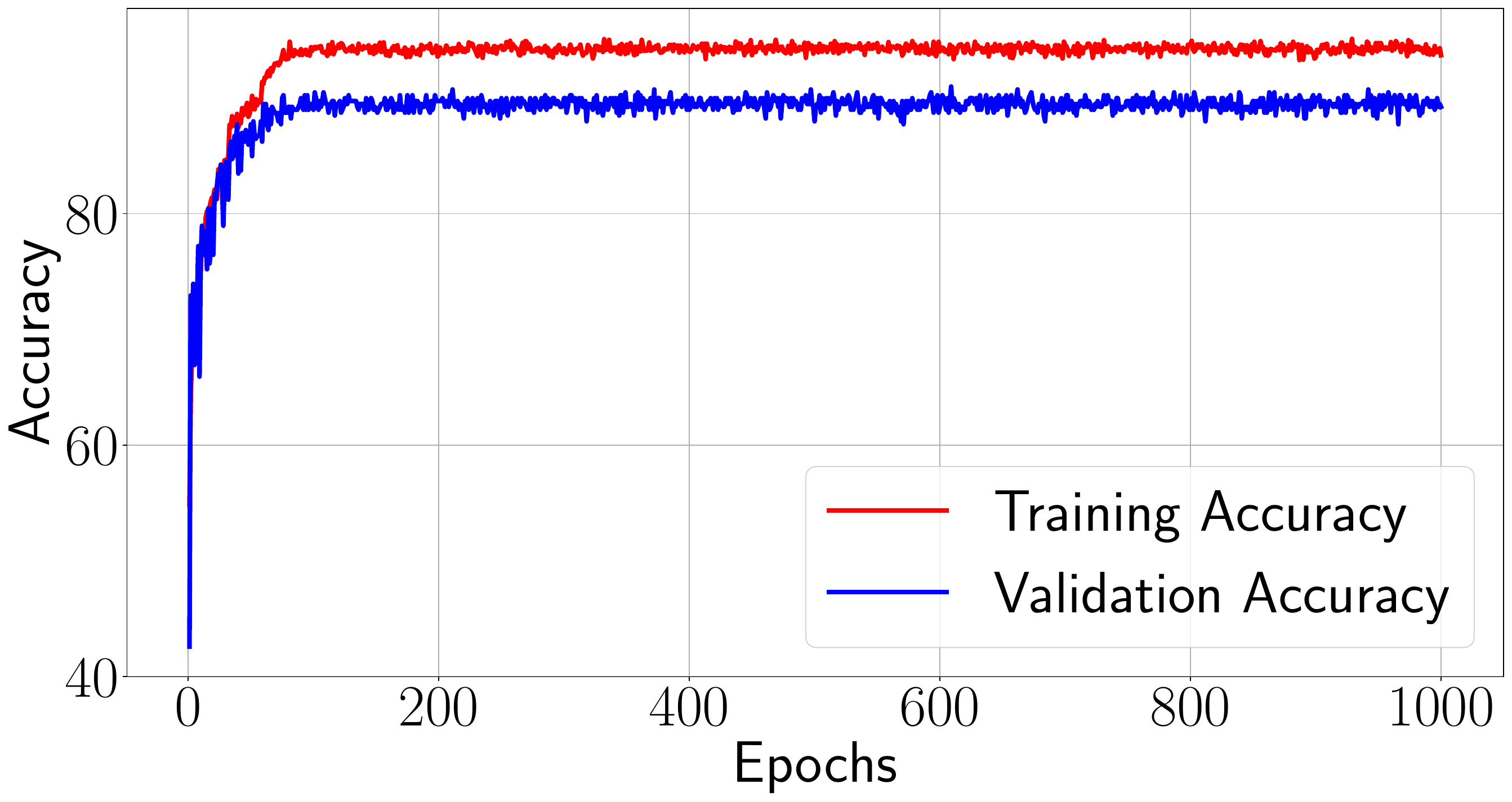}
    \caption{Training and Validation accuracy trend while training the model for 1000 epochs }
    \label{fig:accuracy_plot}
  \end{subfigure}
  \hfill
  \begin{subfigure}[b]{0.49\linewidth}
    \centering
    \includegraphics[width=\linewidth]{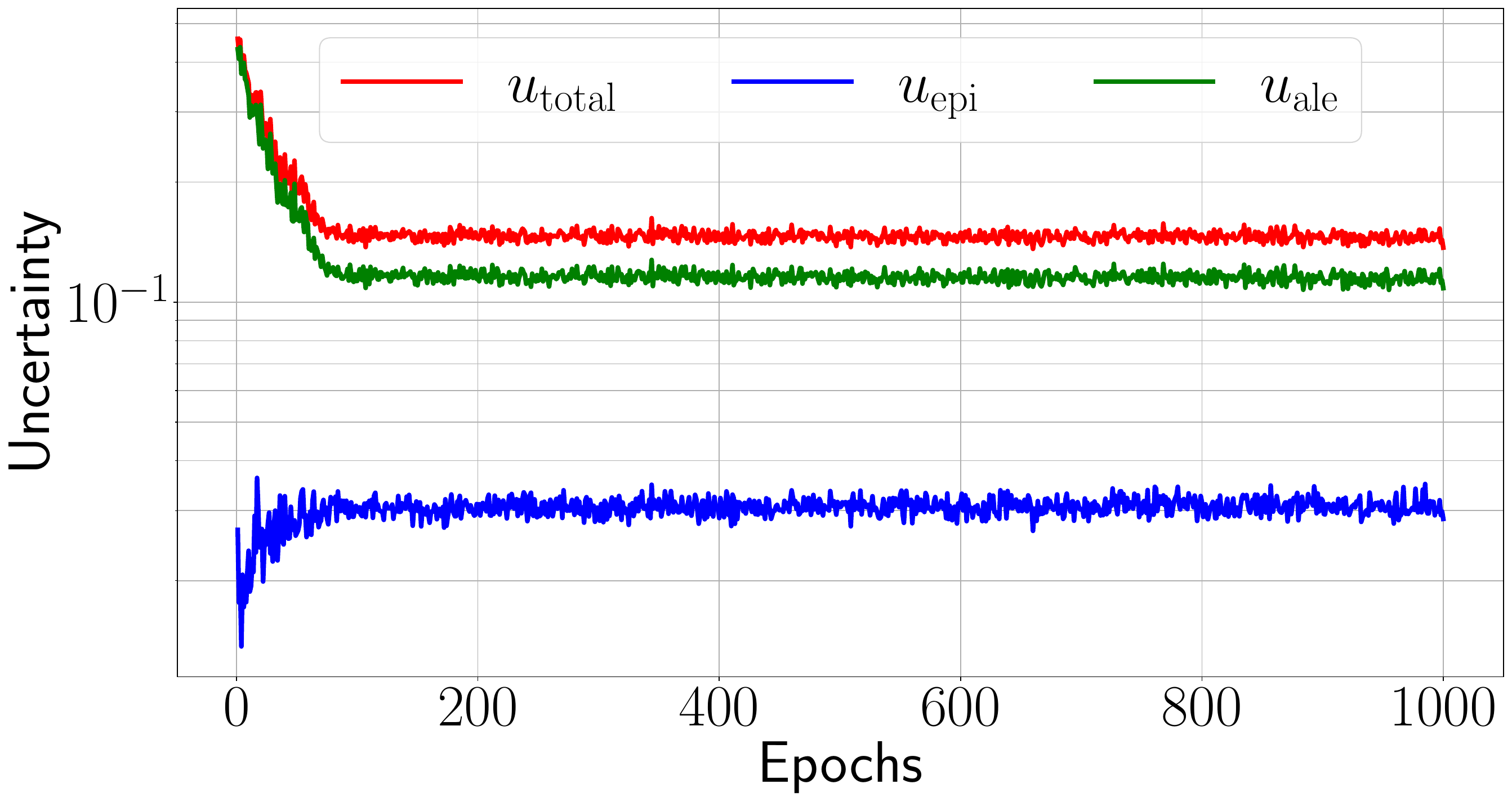}
    \caption{Convergence of total, aleatoric and epistemic uncertainty with training}
    \label{fig:epoch_uncertainty}
  \end{subfigure}
  \caption{(a) Training and Validation accuracy of Bayesian Steerable CNN. We notice that accuracy increasing with epochs and later converges. (b) Converging trend of total $(u_{\mathrm{total}})$, aleatoric $(u_{\mathrm{ale}})$, and epistemic $(u_{\mathrm{epi}})$ uncertainty over 1000 epochs of training.}
  \label{fig:uncertainty_convergence}
\end{figure}

\subsection{Classification Analysis}
Figure~\ref{fig:confusion_matrix} shows the normalized confusion matrix evaluated on the ModelNet10 test data set. For geometrically distinct objects, the models demonstrates strong accuracy per-class (for e.g., sofa, chair, toilet). The geometries of these classes are evidently distinct from each other enabling effective feature space mapping. The class with lowest accuracy is desk, which is presumably being confused with related horizontal flat-surfaced furniture (e.g. bed) as can be seen from Figure~\ref{fig:data_matrix_more_samples}.

We empirically verify Proposition~\ref{prop:equivariance} by measuring the invariance error of the trained model under random $\mathrm{SO}(3)$ rotations of the input, obtaining $\epsilon_{\mathrm{inv}} = 3.41 \times 10^{-5}$. Full details of the verification procedure are given in Appendix~\ref{sec:invariance_verification}.

\begin{figure}[htbp]
  \centering
  \includegraphics[width=0.99\linewidth]{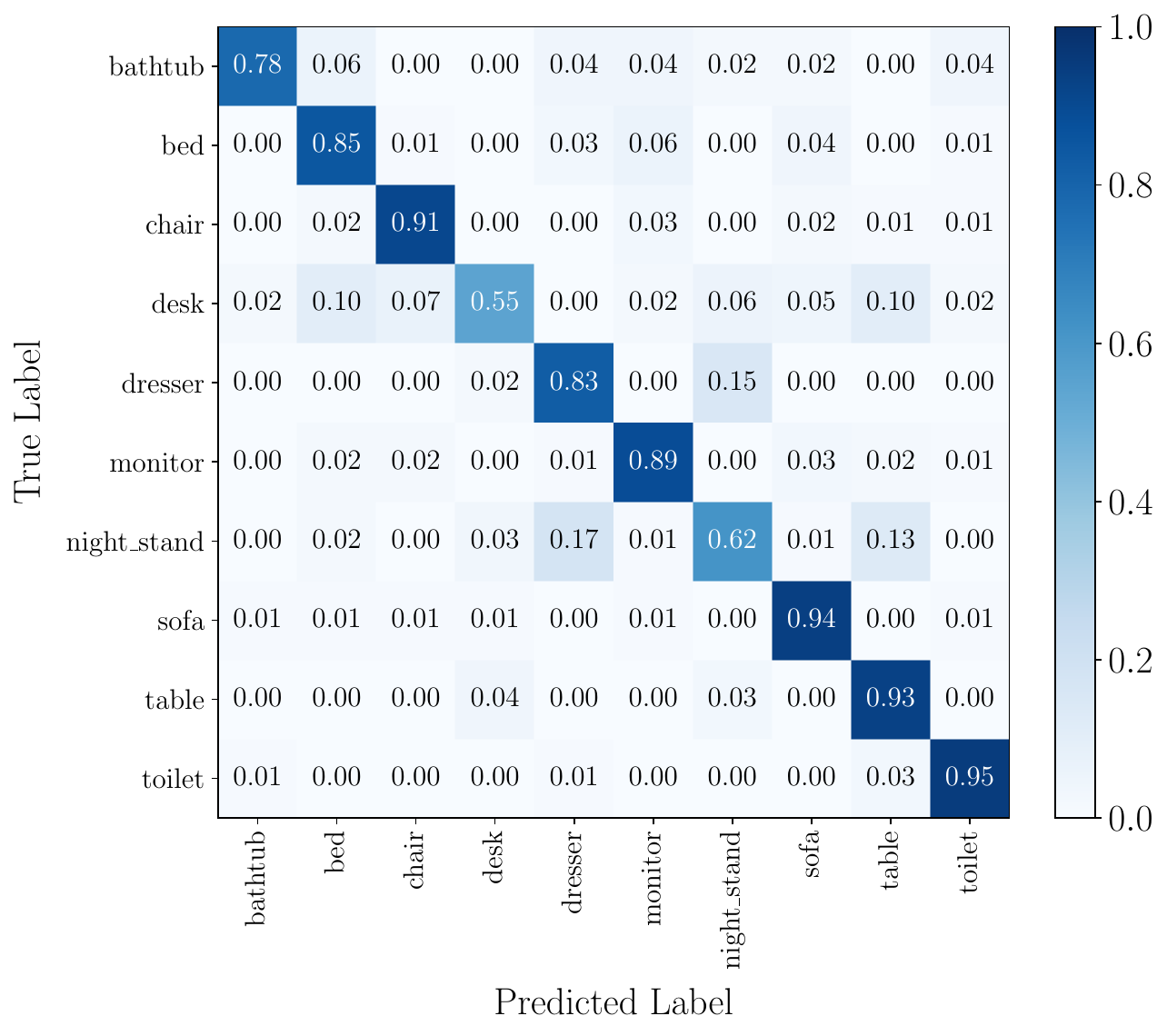}
  \caption{Normalized confusion matrix of the proposed Bayesian steerable CNN on the ModelNet10 test set. The model performs strongly on geometrically distinctive classes such as toilet (0.95), sofa (0.94), and chair (0.91). The primary failure modes are concentrated in desk (0.55) as it gets confused with classes exhibiting geometric resemblance.}
  \label{fig:confusion_matrix}
\end{figure}

\subsection{Generalization}
To evaluate the robustness of the proposed Bayesian 3D Steerable CNN against input distribution shift, both the Bayesian and deterministic baselines are trained exclusively on clean ModelNet10 data (noise level $\epsilon_{\text{noise}} = 0.0$) and evaluated on test sets corrupted with additive Gaussian noise at levels $\epsilon_{\text{noise}} \in {0.0, 0.1, \ldots, 0.9}$. As shown in Figure~\ref{fig:generalization_plot}, both models achieve comparable accuracy at low noise levels ($\epsilon_{\text{noise}} \leq 0.4$), with the deterministic model performing marginally better at $\epsilon_{\text{noise}} = 0.1$. However, as the noise level increases beyond $\epsilon_{\text{noise}} = 0.4$, the Bayesian model exhibits markedly superior robustness: at $\epsilon_{\text{noise}} = 0.9$, the Bayesian model retains a test accuracy of approximately $81.83\%$, compared to $75.66\%$ for the deterministic counterpart — a gap of $6.17\%$. This degradation trend is monotonically more pronounced for the deterministic model, whereas the accuracy Bayesian model decreases at a substantially lower rate across the noise spectrum. We attribute this behavior to the weight uncertainty induced by the variational posterior, which acts as an implicit regularizer and prevents overconfident, brittle point-estimate solutions. These results suggest that Bayesian inference in steerable convolutional networks confers meaningful robustness to out-of-distribution perturbations without any exposure to noise in data during training.

\begin{figure}[htbp]
  \centering
  \includegraphics[width=0.99\linewidth]{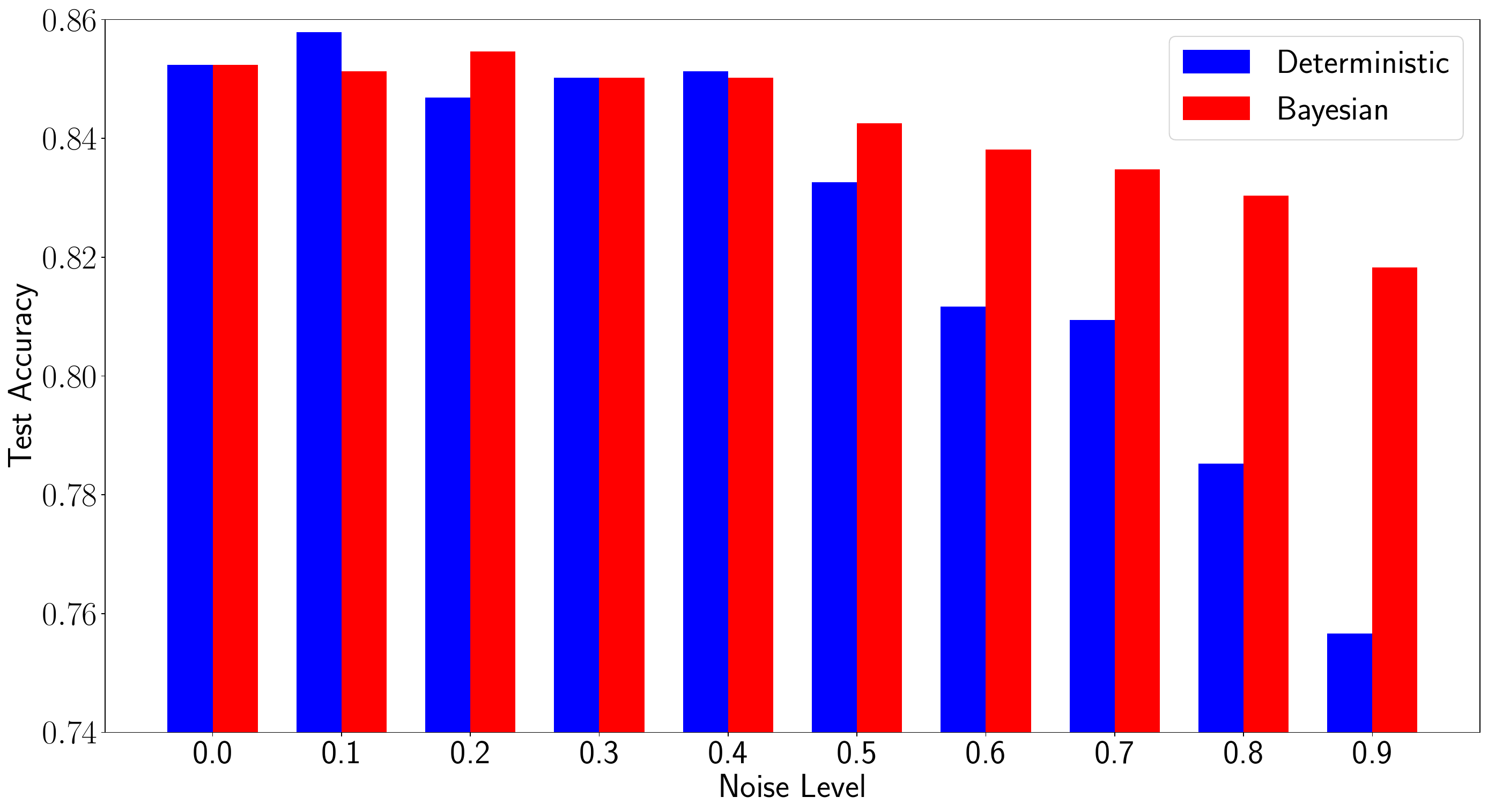}
  \caption{Generalization performance under increasing test data noise. Test accuracy on ModelNet10 as a function of additive Gaussian noise level $\epsilon_{\text{noise}} \in \{0.0, 0.1, \ldots, 0.9\}$,
  with both models trained on clean data ($\epsilon_{\text{noise}} = 0.0$). The Bayesian model degrades more gracefully, outperforming the deterministic baseline by up to $6.17\%$ at $\epsilon_{\text{noise}} = 0.9$.}
  \label{fig:generalization_plot}
\end{figure}

\subsection{Uncertainty Quantification} 
According to Equation~\ref{eq:uncertainty}, we decompose the total predictive uncertainty into its epistemic and aleatoric components, where epistemic uncertainty reflects model uncertainty reducible with more data, and aleatoric uncertainty reflects irreducible noise inherent in the data. Figure~\ref{fig:epoch_uncertainty} shows that total uncertainty decrease steadily over training, converging from approximately 0.4576 to 0.1370, and aleatoric uncertainty decreases too from 0.4308 to 0.1084 by epoch 1000, as the model refines its posterior over the basis coefficients. Crucially, the epistemic uncertainty remains consistently low throughout training ($u_{\mathrm{epi}} \in [0.0136, 0.0362]$), indicating that the model is confident in its learned weight distribution and that the dominant source of uncertainty in the predictions is aleatoric rather than epistemic. 

\begin{figure}[htbp]
  \centering
  \includegraphics[width=0.8\linewidth]{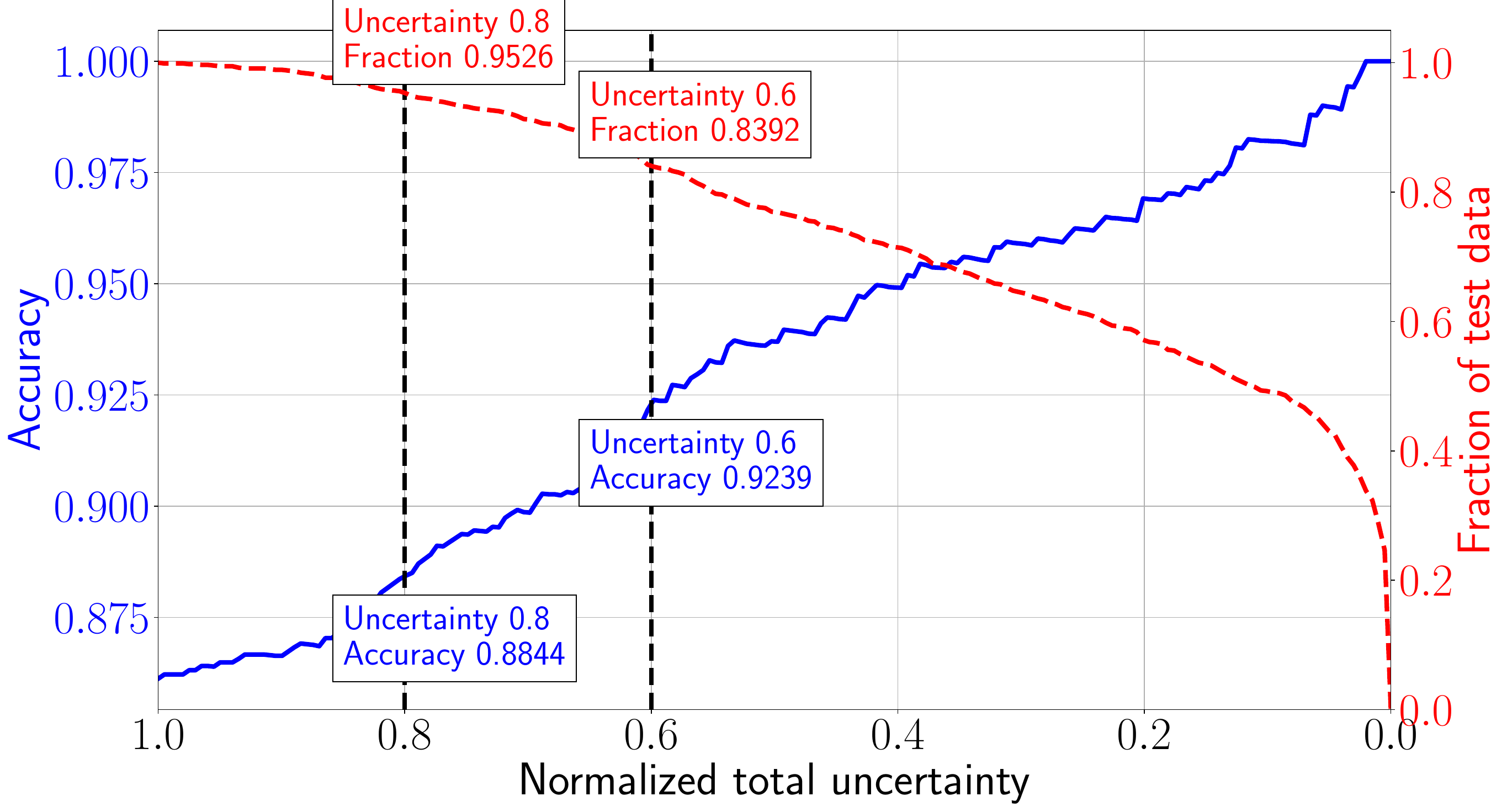}
  \caption{Performance improvements for a subset of test data that has
uncertainty less than the uncertainty–threshold (0.6 and 0.8) represented by vertical dotted lines.}
  \label{fig:threshold_uncertainty}
\end{figure}

In Figure~\ref{fig:threshold_uncertainty}, we estimate the reliability of the uncertainty estimates of the model by analyzing classification accuracy on low uncertainty (high-confidence) subsets of the test data. During the inference on the test dataset using the trained model, the total uncertainty of each test sample is normalized using the minimum and maximum total uncertainty values in the test set. For the marked uncertainty threshold, the test dataset is partitioned into two subsets with one subset having uncertainty lower than the threshold and the other with uncertainty above the threshold. The fraction of test data samples present in the lower uncertainty subset and the corresponding classification accuracy are plotted against the normalized total uncertainty in Figure~\ref{fig:threshold_uncertainty}. When the threshold is tightened, the retained fraction of data in low-uncertainty subset decreases while the accuracy improves. This demonstrates that the total uncertainty estimates of the model are well-calibrated and identify reliably the more confidently predicted samples.

\subsection{Reliability Diagram}
To assess the reliability of predictive uncertainty, we first grouped the test dataset into 15 equal-width confidence bins. We then evaluate the posterior predictive probabilities of outputs in a bin and plot them against the mean bin accuracy.
 The model achieves an Expected Calibration Error (ECE) of 0.0263, indicating strong alignment between predicted confidence and observed accuracy without any additional calibration procedure. The reliability diagram as shown in Figure \ref{fig:sample_reliability_diagram} reveals a mild but consistent underconfidence in the intermediate confidence regime $\text{conf}(B_{7-11}) \in [0.40, 0.733]$, where the accuracy of individual bins exceeds the stated confidence of the model by up to 9.3\% points. In the high-confidence regime: $\text{conf}(B_{14-15}) > 0.87$, which accounts for approximately 64\% of all test samples, the model is slightly overconfident by at most 2.7\% points which is a substantially smaller deviation. The confidence histogram confirms that the majority of predictions concentrate in the rightmost bin $B_{15}$ with $\text{conf}(B_{15}) > 0.933$ and $|B_{15}| = 493$, consistent with a well-trained classifier operating largely in the high-certainty regime. Altogether, these results demonstrate that the Bayesian treatment of the steerable basis coefficients via Bayes-by-Backpropagation yields not only competitive classification accuracy but also well-calibrated uncertainty estimates, supporting the trustworthiness of the model's confidence outputs for downstream decision-making.

\subsection{Uncertainty-Accuracy Correspondence}
To validate that the epistemic uncertainty estimates produced by our Bayesian G-Steerable CNN are semantically meaningful, we examine the relationship between binned normalized epistemic uncertainty and per-bin classification accuracy across the ModelNet10 test set. The epistemic uncertainty $u_{\mathrm{epi}}$ is normalized using min-max normalization scheme as $\hat{u}_{\mathrm{epi}}=\frac{u_{\mathrm{epi}} - \text{min}(u_{\mathrm{epi}})}{\text{max}(u_{\mathrm{epi}}) - \text{min}(u_{\mathrm{epi}})}$. Figure \ref{fig:sample_unc_vs_acc_diagram} presents this analysis, with the count histogram confirming the heavy concentration of probability mass at low uncertainty values. The distribution is strongly right-skewed: 58.7\% of test samples reside in the lowest uncertainty bin $\hat{u}_{\mathrm{epi}}(B_{1}) \in [0, 0.067]$ with mean epistemic uncertainty in bin $B_1$ as $\langle \hat{u}_{\mathrm{epi}}(B_{1}) \rangle = 0.0084$, achieving a per-bin accuracy $\text{acc}(B_1)$ well above the overall test accuracy of 86.3\%. A clear monotonically decreasing trend in accuracy is observed across the low-to-mid uncertainty range ($\hat{u}_{\mathrm{epi}}(B_{1-8})$ $\leq$ 0.53), with accuracy declining from 96.44\% to 55.17\%, consistent with the model assigning higher epistemic uncertainty to ambiguous inputs. The Spearman rank correlation between normalized epistemic uncertainty and per-sample correctness -0.3989, confirming a statistically significant negative monotonic relationship. We note that two bins $B_{12}$ and $B_{15}$ in the high-uncertainty regime $\hat{u}_{\mathrm{epi}}(B_{12-15}) \geq 0.733$ contain very few samples, $|B_{12}|=2, \; |B_{15}|=1$, and exhibit non-monotonic accuracy values, which we attribute to small-sample noise rather than a breakdown of the uncertainty signal. Collectively, these results demonstrate that the posterior variance estimated via Bayes-by-Backpropagation over the steerable convolutional weights reflects prediction reliability: the model assigns higher uncertainty to the samples with higher classification errors.

\begin{figure}[htbp]
  \centering
  \begin{subfigure}[b]{0.49\linewidth}
    \centering
    \includegraphics[width=\linewidth]{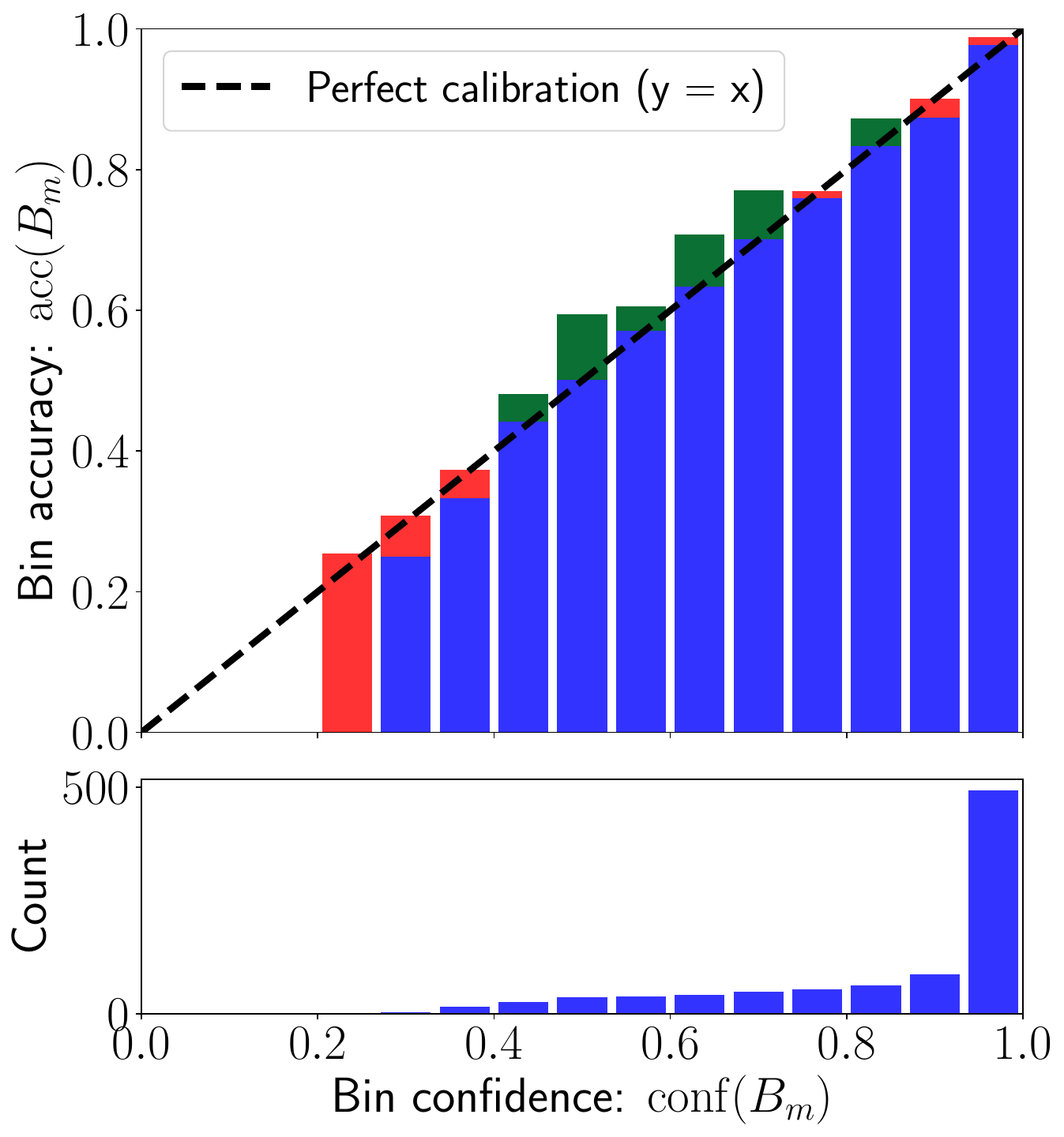}
    \caption{Reliability diagram of the proposed model on ModelNet10 dataset}
    \label{fig:sample_reliability_diagram}
  \end{subfigure}
  \hfill
  \begin{subfigure}[b]{0.49\linewidth}
    \centering
    \includegraphics[width=\linewidth]{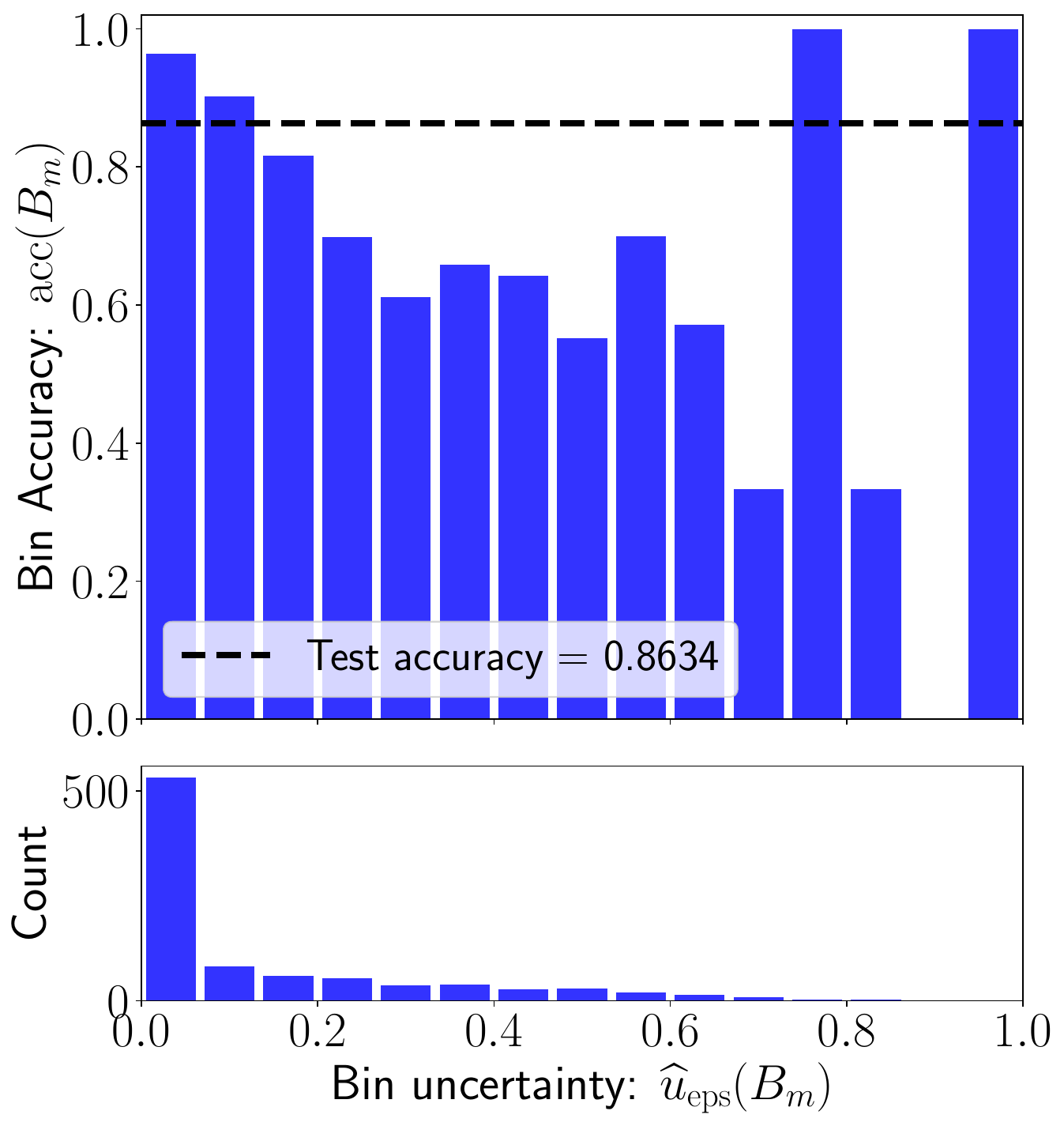}
    \caption{Uncertainty-Accuracy Correspondence of the proposed model on ModelNet10 dataset}
    \label{fig:sample_unc_vs_acc_diagram}
  \end{subfigure}
  \caption{\textbf{(a)} Reliability diagram of the proposed model on the ModelNet10 test set (N = 908), showing strong calibration (ECE = 0.0263). \textcolor{red}{\textbf{Red}} corresponds to overconfidence gap from perfect calibration. \textcolor[HTML]{0A7033}{\textbf{Green}} corresponds to underconfidence gap from perfect calibration. \textbf{(b)} Bin Accuracy vs. Normalized epistemic uncertainty for our model on ModelNet10 dataset. Each bar shows the fraction of correct predictions within a normalized epistemic uncertainty bin; the count histogram (bottom) reflects the strong concentration of test samples at low uncertainty.}
  \label{fig:uq_validate}
\end{figure}

\vspace{-10pt}
\section{Limitations}\label{sec:Limitations}
The proposed Bayesian framework doubles the number of parameters, leading to higher computational cost than deterministic steerable CNNs. In addition, specifying suitable priors for unseen datasets can be challenging and may influence results. Finally, as the model effectively averages over multiple deterministic networks, it may exhibit slightly lower predictive accuracy than its deterministic counterpart. Steerable CNNs are parameter-efficient but more computationally demanding than standard CNNs, and may therefore be unsuitable for applications where equivariance is not essential. The steerable basis grows rapidly with the polynomial order of the field type and the size of the group. For higher-order representations, the steerable CNNs can become very large, making the Bayesian layer significantly more expensive than its deterministic counterpart as adding stochasticity doubles the parameters of the model.

\section{Conclusions}
We have proposed a Bayesian formulation of 3D Steerable-CNNs that preserves SE(3)-equivariance exactly while providing principled uncertainty quantification. Stochastic equivariant kernels are parameterized as linear combinations of deterministic steerable basis functions with stochastic coefficients, whose posterior distributions are inferred via variational inference. This provides a principled approach to learning distributions over symmetry-constrained kernels, enabling simultaneous uncertainty quantification and SE(3) equivariance in 3D CNNs. Empirical evaluation demonstrates that the proposed Bayesian Steerable-CNN attains classification accuracy competitive with its deterministic counterpart. Uncertainty-guided selective prediction yields an approximately 4\% improvement in classification accuracy on 84\% of the test set, demonstrating that the learned uncertainty is meaningful. The Bayesian formulation further admits a principled decomposition of predictive uncertainty into epistemic and aleatoric components, enabling diagnostic attribution of error sources and facilitating targeted refinement of the model. Under distributional shift induced by additive Gaussian noise of increasing severity, the proposed model exhibits markedly improved robustness over its deterministic counterpart, retaining a 6.17\% accuracy advantage at 90\% noise intensity. Reliability analysis confirms well-calibrated predictive probabilities without recourse to any post-hoc calibration procedure. Moreover, a statistically significant negative Spearman rank correlation between normalized epistemic uncertainty and per-bin mean accuracy establishes that the learned posterior variance captures statistically meaningful signal about prediction reliability — samples assigned higher epistemic uncertainty exhibit systematically larger classification error — validating the model's uncertainty estimates as a principled basis for risk-aware decision-making.

\section{Acknowledgments}
This work was supported by National Science Foundation under Grant No. 2442313. This work used Bridges-2 at Pittsburgh Supercomputing Center through allocation CIS260992 from the Advanced Cyberinfrastructure Coordination Ecosystem: Services \& Support (ACCESS) program, which is supported by U.S. National Science Foundation grants \#2138259, \#2138286, \#2138307, \#2137603, and \#2138296. We have also used MRI/DeepBlizzard GPU cluster computing facility at MTU. The authors would like to thank Dr. Gabriele Cesa and Dr. Maurice Weiler for their generous assistance via email correspondence in clarifying various technical aspects of the \texttt{escnn} library.

\clearpage
\bibliographystyle{plainnat}
\bibliography{refs}

\clearpage
\appendix
\appendix

\section*{Appendix}
\addcontentsline{toc}{section}{Appendix}

\section{Architecture Details}
\label{sec:appendix_arch_details}

\begin{figure}[b!]
    \centering
    \includegraphics[width=0.99\linewidth]{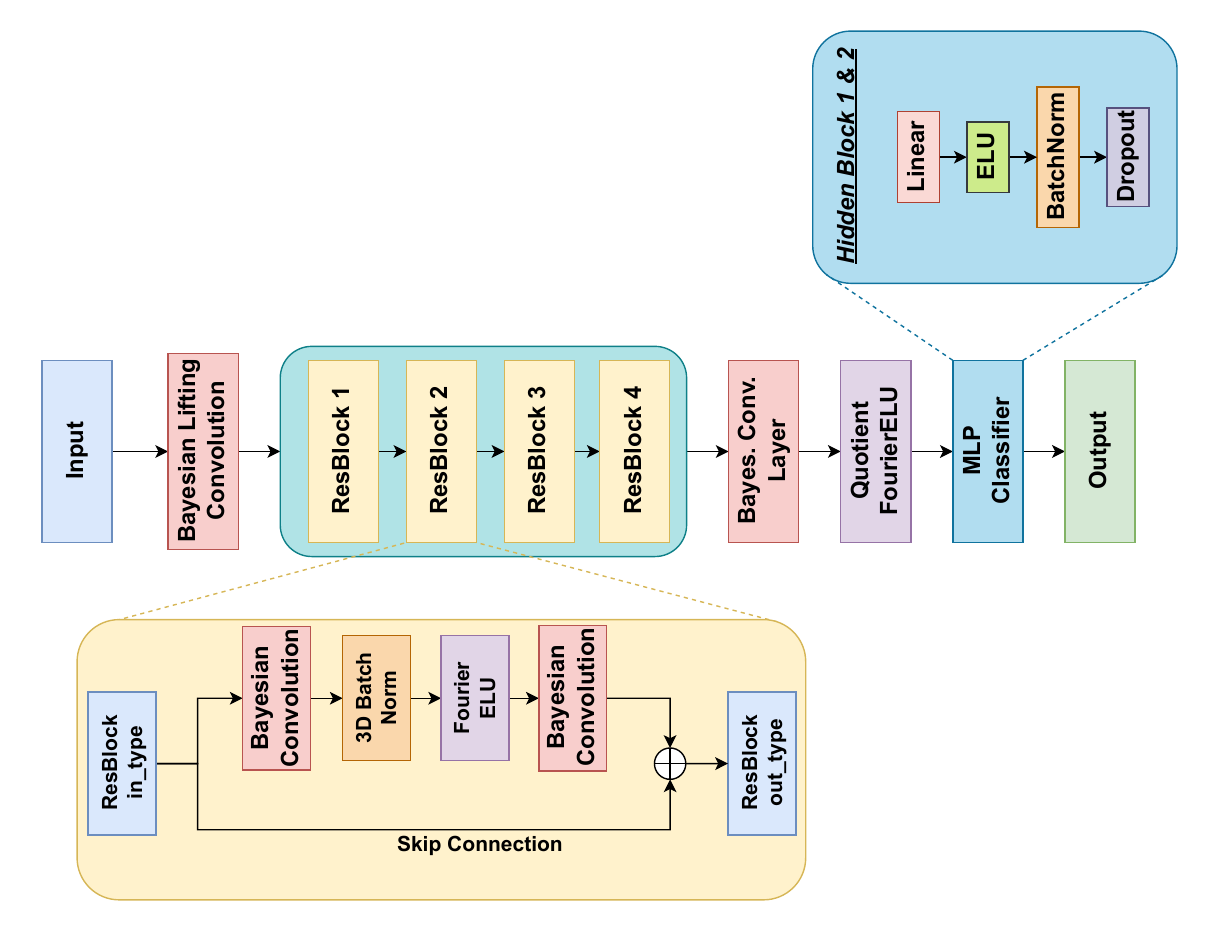}
    \caption{Full Bayesian SE(3)-equivariant architecture. Every \texttt{R3Conv} layer in the steerable backbone of~\citet{cesa2022program} is replaced with a \texttt{BayesianR3Conv} that places a mean-field Gaussian posterior over the steerable basis coefficients. The final invariant feature vector is classified by a deterministic three-layer MLP}
    \label{fig:model_architecture}
\end{figure}

The full SE(3)-equivariant Bayesian residual architecture used in all experiments is shown in Figure~\ref{fig:model_architecture}. The network takes a scalar voxel grid of shape $(B, 1, D, H, W)$, interpreted as a trivial-representation field, and first applies a lifting \texttt{BayesianR3Conv} with kernel size~5 and padding~2 that maps it to the steerable field type $(\rho_2^{\mathrm{poly}})^{\times 3}$, where $\rho_2^{\mathrm{poly}} = \rho_0 \oplus \rho_1 \oplus \rho_2$ has dimension $1 + 3 + 5 = 9$. Four residual blocks then progressively deepen the representation through output field types $(\rho_3^{\mathrm{poly}})^{\times 2}$, $(\rho_3^{\mathrm{poly}})^{\times 6}$, $(\rho_3^{\mathrm{poly}})^{\times 12}$, and $(\rho_3^{\mathrm{poly}})^{\times 8}$, with width parameters 200, 480, 480, and 960 respectively, where $\rho_3^{\mathrm{poly}} = \rho_0 \oplus \rho_1 \oplus \rho_2 \oplus \rho_3$ has dimension 16. Each block halves the spatial resolution via a stride of 2.

Each residual block applies the sequence \texttt{BayesianR3Conv} ($3{\times}3{\times}3$, padding~1) $\to$ \texttt{IIDBatchNorm3d} (affine) $\to$ \texttt{FourierELU} (Thomson-cube sampling of $\mathrm{SO}(3)$ with $S=96$ points, band-limited to $L=2$) $\to$ \texttt{BayesianR3Conv} ($3{\times}3{\times}3$, padding~1, stride~2 when downsampling). Skip connections use \texttt{PointwiseAvgPoolAntialiased3D} (scale 0.33, stride~2) followed by a $1{\times}1{\times}1$ \texttt{BayesianR3Conv} when spatial or channel dimensions change, and the identity otherwise.

A final \texttt{BayesianR3Conv} (kernel 3, no padding) maps to the pooling layer's input type, and a \texttt{QuotientFourierELU} samples 128 feature fields of band-limited type ($L=2$) on the 24-point snub-cube grid of $S^2$~\citep{cesa2022program}, retaining only the $\ell=0$ (rotationally invariant) component to produce an invariant feature vector $\mathbf{z} \in \R^{128}$. This is classified by a deterministic three-layer MLP.



\section{ModelNet10 data samples}
ModelNet10~\citep{wu20153d} is a subset of the broader ModelNet40 benchmark, comprising 4,899 CAD models spanning 10 object categories: bathtub, bed, chair, desk, dresser, monitor, night stand, sofa, table, and toilet. The dataset provides an official 80\%--20\% train/test split, yielding 3,991 training and 908 test samples. Each model is represented as a polygon mesh, which we voxelize into a $33\times33\times33$ occupancy grid to obtain a scalar field suitable for input to the SE(3)-equivariant architecture. 
The further division for training, validation, and testing are as discussed in Sec .~\ref {subsec:modelnet10_classification}
The dataset is moderately class-imbalanced, with categories such as nightstand and dresser being underrepresented relative to chair and sofa. A few voxelized samples from each category are shown in Figure~\ref{fig:data_matrix_more_samples}.

\begin{figure}[htbp!]
    \centering
    \includegraphics[width=1\linewidth]{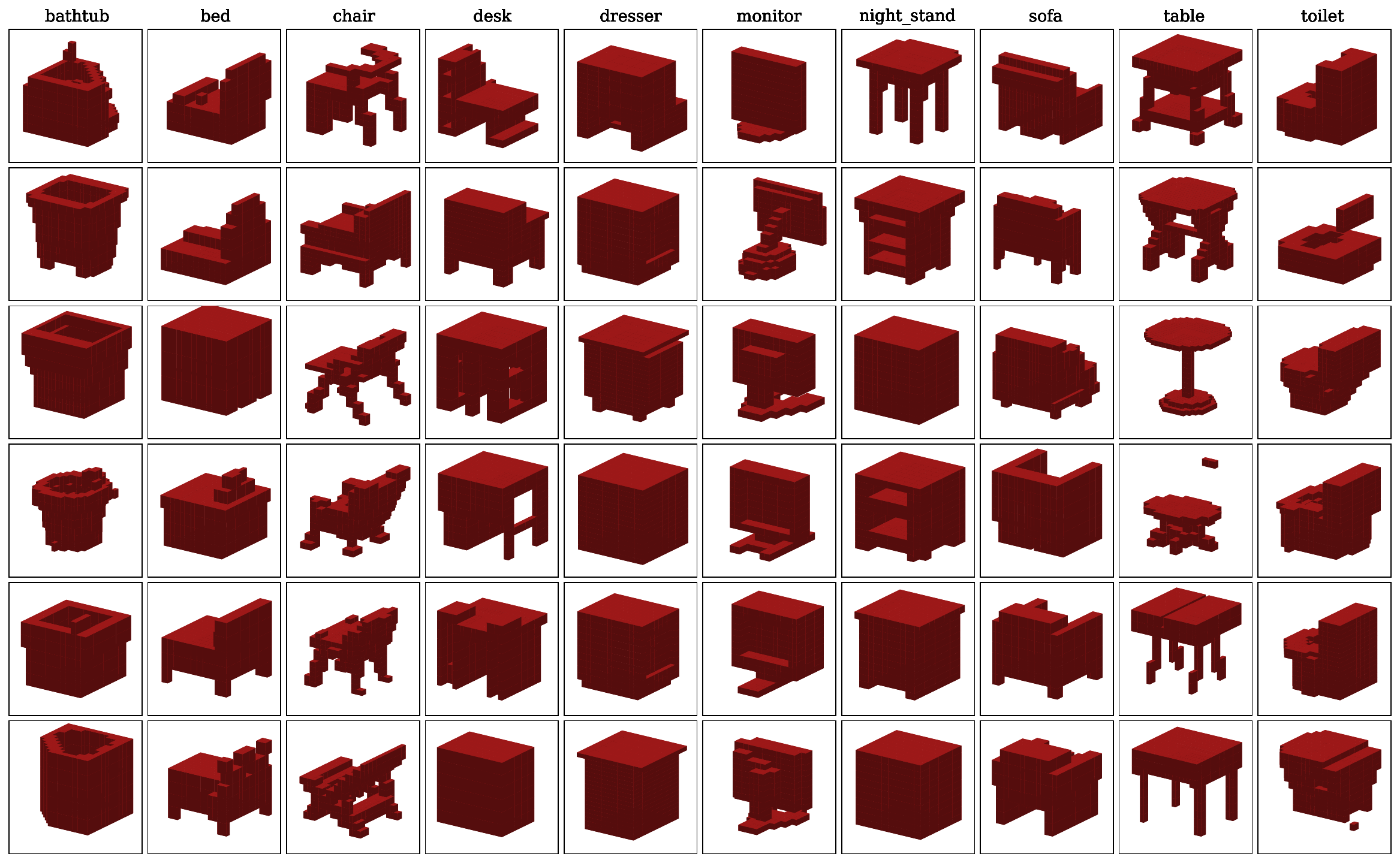}
    \caption{Sample voxelized 3D objects from the ModelNet10 dataset, with each column representing a category of a sample.}
    \label{fig:data_matrix_more_samples}
\end{figure}

\section{Invariance Verification} \label{sec:invariance_verification}
The proposed architecture is designed to be equivariant layer-by-layer throughout the steerable convolutional backbone, where intermediate feature fields transform predictably under the action of $\mathrm{SO}(3)$ according to their assigned representation type $\rho$. This layer-wise equivariance is then deliberately collapsed to invariance at the final pooling layer, Quotient Fourier ELU, which retains only the $\ell = 0$ trivial representation component of the feature field --- the rotationally invariant subspace --- and discards all higher-order equivariant components. The resulting 128-dimensional invariant vector is then passed to a deterministic three-layer MLP for classification, whose output class probabilities are therefore invariant to any rotation of the input by construction. The full pipeline thus realises a principled progression: $\mathrm{SO}(3)$-equivariant feature extraction in the convolutional layers, followed by symmetry reduction to invariance at the pooling layer, followed by invariant classification.

To empirically verify that this invariance is preserved under Bayesian weight sampling, we measure the invariance error, defined as the change in the predicted class probabilities when the input is rotated as stated in Eq. \ref{eq:equiv_error}. 

For a weight sample $\mathbf{w}^{(s)} \sim q$ drawn from the trained variational posterior, the measured invariance error is $\epsilon_{\mathrm{inv}} = 3.41 \times 10^{-5}$, with the numerical errors arising from floating-point arithmetic and the discrete voxel grid approximation of the continuous group action. This confirms that stochasticity in the steerable basis coefficients --- the defining feature of our Bayesian extension --- does not compromise the rotational invariance of the final classification output, as guaranteed theoretically by Proposition~\ref{prop:equivariance}.

\end{document}